\documentclass[10pt,letterpaper]{article}
\usepackage[letterpaper,margin=1in]{geometry}
\usepackage[T1]{fontenc}
\usepackage[utf8]{inputenc}
\usepackage{amsmath,amsthm,mathtools}
\usepackage{amssymb}
\usepackage{newtxtext,newtxmath}
\usepackage{graphicx,booktabs,array,tabularx}
\usepackage[table,dvipsnames]{xcolor}
\usepackage{enumitem,microtype}
\usepackage[font=small,labelfont=bf]{caption}
\usepackage{url}
\usepackage[colorlinks=true,citecolor=blue,linkcolor=blue,urlcolor=blue]{hyperref}
\usepackage{bookmark}

\newtheorem{proposition}{Proposition}

\theoremstyle{remark}

\theoremstyle{definition}

\newcommand{\R}{\mathbb{R}}

\newcommand{\Rot}{\mathcal{R}}

\newcommand{\SO}{\mathrm{SO}}

\newcommand{\Sp}{\mathrm{Sp}}
\DeclareMathOperator{\diag}{diag}

\setlist{nosep,topsep=4pt}
\graphicspath{{figs/}}

\hypersetup{
    colorlinks=true,
    linkcolor=blue,
    citecolor=blue,
    urlcolor=blue
}
\title{\textbf{HySPE: Positional Encoding via Symplectic Dual Shears}}
\author{\\ Zhongping Ji}
\date{}

\begin{document}
\maketitle

\begin{abstract}
We introduce \textbf{Hyperbolic Symplectic Positional Encoding (HySPE)}, grounding positional attention in non-compact symplectic transformations. While canonical Rotary Position Embedding (RoPE) parameterizes the compact, elliptic branch of $\Sp(2,\R)$ via rotations, HySPE operationalizes its hyperbolic branch via a damped symmetric composition of dual shears, yielding a conformally symplectic contraction with two spectral decay rates per channel pair. To eliminate the exponential representation drift inherent to naive absolute factorizations, we diagonalize the operator in its invariant eigenbasis and introduce blockwise coordinate rebasing with adaptive centered execution. This guarantees length-independent numerical bounds while matching cached RoPE forward latency (7.21\,ms on an RTX 4090). On TinyShakespeare, HySPE-UltraLong maintains an invariant perplexity of 4.810 up to $16\times$ zero-shot extrapolation ($L=4096$), whereas RoPE degrades to 131.198. Scaled to a 51M-parameter subword Transformer on WikiText-103 ($L_{\text{train}}=512$), HySPE closely matches RoPE in-domain while robustly extrapolating to length 8192, reducing tail perplexity by 83.9\% over RoPE. While these controlled experiments establish HySPE's extrapolation robustness and numerical stability, evaluating its scaling behavior on large-scale foundation models remains an important direction for future investigation.
\end{abstract}

\section{Introduction}
Positional encodings dictate how self-attention mechanisms in Transformers~\cite{vaswani2017attention}
represent relative token displacements.
Rotary Position Embedding (RoPE)~\cite{su2021roformer} establishes relative-position dependence through pairwise orthogonal rotations of query and key projections. In two dimensions, these rotations correspond to the elliptic subgroup of the real symplectic group $\Sp(2,\R) \cong \mathrm{SL}(2,\R)$. While volume-preserving and isometric, elliptic orbits are bounded and non-dissipative; consequently, extending sequence length beyond the pretraining horizon exposes unobserved phase domains, precipitating phase aliasing and attention entropy collapse~\cite{kazemnejad2023impact,chen2023extending}.

In this work, we propose Hyperbolic Symplectic Positional Encoding (HySPE), exploring the complementary hyperbolic branch of $\Sp(2,\R)$. Motivated by dissipative dynamical systems and hyperbolic geometric flows~\cite{ananke2026,minkowski2026}, HySPE constructs relative position operators via symmetric dual shears. The composition of two transverse elementary shears generates a symmetric hyperbolic operator with reciprocal positive real eigenvalues. When modulated by an exponential decay envelope, the relative attention kernel becomes contractive and conformally symplectic, satisfying $T_\Delta^\top J T_\Delta = e^{-2\gamma\Delta} J$ for the standard symplectic form $J$. This endows each channel pair with two distinct spectral attenuation rates, providing anisotropic distance filtering directly within the representation space.

Crucially, we uncover and resolve a fundamental tension between the continuum algebraic kernel and its finite-precision realization. Although the relative operator $T_\Delta$ is strictly bounded, expressing it via separated absolute query and key matrices forces the key operator to diverge exponentially as $e^{(2\chi+\delta)n}$, causing severe numerical overflow in float precision and corrupting KV-cache representation scales. 

We overcome this representation divergence without modifying the mathematical kernel. By diagonalizing the dual-shear operator into its invariant eigenbasis, we factorize the bilinear interaction into decoupled decay modes and introduce \emph{spectral block rebasing}. This local gauge transformation anchors intermediate scales to query-block origins, bounding maximum positive exponents to small constants regardless of sequence length. To eliminate kernel dispatch overhead, we design an adaptive centered blocking scheme that executes long-sequence attention in few kernel calls.

Across rigorous, multi-seed experiments on a byte-level Transformer trained on TinyShakespeare at length 256, HySPE achieves marked in-domain perplexity gains and zero-shot extrapolation up to 4096 tokens ($16\times$ extension) with zero quality degradation. Identical-checkpoint interventions verify that the stable spectral formulation completely recovers the predictive capacity lost to floating-point truncation in naive factorization. Furthermore, optimized HySPE forward latency reaches 7.21\,ms at length 4096, matching cached RoPE (7.20\,ms) on an RTX 4090. HySPE demonstrates that hyperbolic symplectic geometry provides an expressive, numerically robust alternative to canonical rotations for foundation sequence models.

\section{Related Work}

\paragraph{Rotary, scaled, and multi-rate position representations.}
In Transformer self-attention~\cite{vaswani2017attention}, relative position mechanisms modulate query-key dot products through relative displacement $\Delta = m - n$. RoPE~\cite{su2021roformer} enforces this via 2D orthogonal rotations, corresponding to the elliptic subgroup of $\Sp(2,\R)$, which preserves vector $L_2$ norms but lacks intrinsic distance damping and exhibits periodic phase oscillation under context extension. xPos~\cite{sun2022length} augments rotary transformations with empirical exponential damping $\zeta^\Delta \Rot(\Delta\theta)$ to facilitate length extrapolation. However, xPos applies an isotropic scalar decay that uniformly contracts circular orbits. HySPE fundamentally departs from the elliptic branch: it parameterizes the hyperbolic branch of $\Sp(2,\R)$ via symmetric dual shears, producing a contractive spectrum with two unequal real decay rates per channel pair.

\paragraph{Attractor dynamics, hyperbolic geometry, and Lorentz boosts.}
The formulation of HySPE is motivated by recent advances in non-Euclidean representation learning, notably contractive torus attractors in Ananke~\cite{ananke2026} and closed-form hyperbolic flows in Minkowski Attractor Networks~\cite{minkowski2026}. In language modeling, HoPE~\cite{dai2025hope} recently explored continuous Lorentz rotations in hyperbolic geometry to enforce monotonic attention decay. However, parameterizing continuous Lorentz boosts in Euclidean attention without coordinate stabilization introduces severe representation divergence (exponential key-norm explosion) over long contexts. In contrast, HySPE establishes an exact, discrete relative positional kernel via dual shears that preserves conformal symplectic invariants, analytically resolving representation drift through invariant eigenbasis rebasing and centered block execution.

\paragraph{Decayed and recurrent attention.}
RetNet~\cite{sun2023retentive} and ALiBi~\cite{press2021train} introduce scalar distance penalties into linear retention and softmax attention, respectively. In contrast to scalar attention biases that penalize all channels indiscriminately, HySPE operates directly on 2D feature subspaces via hyperbolic shears, providing content-dependent, anisotropic decay. Moreover, HySPE retains exact quadratic causal attention rather than relying on linear-attention approximations.

\paragraph{Exact attention and memory-efficient execution.}
FlashAttention~\cite{dao2022flashattention} highlights the primacy of tiling and IO awareness in Transformer efficiency. Our stable implementations operate within standard scaled dot-product attention (SDPA), optimizing positional coordinate scales and block scheduling to match native rotary efficiency on modern hardware.

\section{Symplectic Dual-Shear Attention}

\subsection{Symplectic Group Perspective and Dual-Shear Construction}
Linear transformations on two-dimensional phase space preserving the canonical symplectic form $J = \begin{pmatrix} 0 & 1 \\ -1 & 0 \end{pmatrix}$ constitute the real symplectic group:
\begin{equation}
 \Sp(2,\R) = \left\{ M \in \R^{2\times2} \;\middle|\; M^\top J M = J \right\} \cong \mathrm{SL}(2,\R).
\end{equation}
Elements of $\Sp(2,\R)$ fall into three conjugacy classes based on their trace $\mathrm{Tr}(M)$:
\begin{enumerate}[leftmargin=1.8em]
    \item \textbf{Elliptic ($|\mathrm{Tr}(M)| < 2$):} Conjugate to standard Euclidean rotations $\SO(2)$. The eigenvalues reside on the unit circle $e^{\pm i\theta}$. This class underpins RoPE~\cite{su2021roformer}. Its bounded orbits lack geometric dissipation.
    \item \textbf{Parabolic ($|\mathrm{Tr}(M)| = 2$):} Unipotent shears of the form $\begin{pmatrix} 1 & h \\ 0 & 1 \end{pmatrix}$. They represent degenerate, shear-like transitions without spectral separation.
    \item \textbf{Hyperbolic ($|\mathrm{Tr}(M)| > 2$):} Real reciprocal eigenvalues $e^{\pm \chi}$. They induce distinct expansive and contractive axes, defining hyperbolic boosts.
\end{enumerate}

HySPE investigates the hyperbolic symplectic branch. For a parameter $\chi \geq 0$, let $h = 2\sinh(\chi/2)$. We construct the fundamental operator $H$ via a symmetric composition of two elementary parabolic shears:
\begin{equation}
 H = \begin{pmatrix} 1 & h \\ 0 & 1 \end{pmatrix}
     \begin{pmatrix} 1 & 0 \\ h & 1 \end{pmatrix}
   = \begin{pmatrix} 1 + h^2 & h \\ h & 1 \end{pmatrix}.
 \label{eq:shear}
\end{equation}
Direct calculation verifies that $\det(H) = (1+h^2) - h^2 = 1$, and consequently:
\begin{equation}
 H^\top J H = J \implies H \in \Sp(2,\R).
\end{equation}
The matrix $H$ is symmetric positive definite with $\mathrm{Tr}(H) = 2 + h^2 = 2\cosh\chi \geq 2$. Its eigenvalues are precisely $e^{-\chi}$ and $e^{\chi}$.

\subsection{Conformally Symplectic Attention Kernel}
To ensure contractive attention over arbitrary distances, we introduce dissipation rate $\gamma = \chi + \delta$, where $\delta \geq 0$. We modulate row-vector query and key representations $q_m, k_n \in \R^{1 \times 2}$ at positions $m$ and $n$ through dual contragredient actions:
\begin{equation}
 \widetilde q_m = e^{-\gamma m} q_m H^m, \qquad
 \widetilde k_n = e^{\gamma n} k_n (H^{-n})^\top.
 \label{eq:absolute}
\end{equation}

For causal displacement $\Delta = m - n \geq 0$, their bilinear inner product satisfies:
\begin{equation}
 \widetilde q_m \widetilde k_n^\top = q_m T_\Delta k_n^\top, \qquad
 T_\Delta = e^{-\gamma\Delta} H^\Delta.
 \label{eq:relative}
\end{equation}

\paragraph{Design dichotomy: Absolute vs.\ relative shear modulations.}
An appealing computational property of elementary shears is that their action $(x, y) \mapsto (x + hy, y)$ involves only a single fused multiply-add (FMA) instruction, bypassing transcendental evaluations ($\cos, \sin, \cosh, \sinh$) entirely. A natural aspiration is to directly substitute RoPE's orthogonal rotations with such unipotent shears. However, direct symmetric application fails to produce a relative kernel. For an elementary shear $S_m = \begin{pmatrix} 1 & mh \\ 0 & 1 \end{pmatrix}$, its non-orthogonality ($S_m^\top \neq S_m^{-1}$) implies that the symmetric bilinear form couples absolute coordinates:
\begin{equation}
 S_m^\top S_n = \begin{pmatrix} 1 & 0 \\ mh & 1 \end{pmatrix} \begin{pmatrix} 1 & nh \\ 0 & 1 \end{pmatrix} = \begin{pmatrix} 1 & nh \\ mh & 1 + mnh^2 \end{pmatrix},
\end{equation}
which explicitly depends on $m \cdot n$ rather than $m - n$. 

This distinction leads to two divergent architectural strategies: (1) \emph{absolute shear modulation}, which exploits ultra-cheap FMA operators while accepting coordinate cross-coupling, and (2) \emph{relative dual-shear modulation}, which restores exact shift invariance $\Delta = m - n$ by pairing a composed symmetric shear $H$ with its contragredient dual $(H^{-n})^\top$. In this paper, we focus strictly on the latter, developing HySPE as a principled relative positional framework.

\begin{proposition}[Conformal Symplectic Property and Contractive Spectrum]
The relative positional transfer operator $T_\Delta$ satisfies the conformal symplectic identity:
\begin{equation}
 T_\Delta^\top J T_\Delta = e^{-2\gamma\Delta} J.
\end{equation}
Furthermore, let $H = U \diag(e^{-\chi}, e^{\chi}) U^\top$ denote the orthogonal eigendecomposition of $H$. Then:
\begin{equation}
 T_\Delta = U \diag\!\left(e^{-(2\chi+\delta)\Delta}, e^{-\delta\Delta}\right) U^\top,
 \qquad \|T_\Delta\|_2 = e^{-\delta\Delta} \leq 1.
 \label{eq:spectral}
\end{equation}
\end{proposition}
\begin{proof}
Since $H \in \Sp(2,\R)$, $H^\top J H = J$, by induction $(H^\Delta)^\top J H^\Delta = J$. Thus:
\begin{equation*}
 T_\Delta^\top J T_\Delta = \left(e^{-\gamma\Delta} H^\Delta\right)^\top J \left(e^{-\gamma\Delta} H^\Delta\right) = e^{-2\gamma\Delta} (H^\Delta)^\top J H^\Delta = e^{-2\gamma\Delta} J.
\end{equation*}
Substituting $H^\Delta = U \diag(e^{-\chi\Delta}, e^{\chi\Delta}) U^\top$ and $\gamma = \chi + \delta$ into Eq.~\eqref{eq:relative} directly yields Eq.~\eqref{eq:spectral}. The operator norm is given by the dominant eigenvalue $e^{-\delta\Delta} \leq 1$.
\end{proof}

For head dimension $d_h$ with $K = d_h / 2$ pairs, the complete attention score evaluates as:
\begin{equation}
 \ell_{mn} = \frac{1}{\sqrt{d_h}} \sum_{j=0}^{K-1} q_{m,j} T_{\Delta,j} k_{n,j}^\top, \qquad
 o_m = \sum_{n \leq m} \frac{e^{\ell_{mn}}}{\sum_{t \leq m} e^{\ell_{mt}}} v_n.
 \label{eq:attention}
\end{equation}
HySPE introduces zero additional learned parameters; the values, projections, and causal softmax retain standard Transformer forms.

\subsection{The Representation Divergence Dilemma}
Proposition 1 reveals a critical numerical dichotomy. While the composite operator $T_\Delta$ is contractive ($\|T_\Delta\|_2 \leq 1$), the naive absolute key factors in Eq.~\eqref{eq:absolute} diverge exponentially. Specifically, the key transformation matrix $T_K(n) = e^{\gamma n} (H^{-n})^\top$ possesses eigenvalues:
\begin{equation}
 \lambda(T_K(n)) \in \left\{ e^{(\gamma+\chi)n}, e^{(\gamma-\chi)n} \right\} = \left\{ e^{(2\chi+\delta)n}, e^{\delta n} \right\}.
\end{equation}
Its spectral radius grows as $\rho(T_K(n)) = e^{(2\chi+\delta)n}$ and its condition number scales as $\kappa(T_K(n)) = e^{2\chi n}$. At $n=4096$, intermediate activations blow up beyond $10^{30}$, triggering immediate exponent overflow in standard 16-bit floating-point formats and severe rounding truncation in FP32. Computing relative scores through separated absolute matrices destroys precision, obscuring the true performance of the underlying mathematical kernel.

\section{Numerically Stable Evaluation of HySPE}

\subsection{Exact Spectral Rebasing}
We resolve representation divergence by eliminating absolute coordinate scaling entirely. Let $r = (r^{(1)}, r^{(2)}) = (2\chi+\delta, \delta)$ denote the paired spectral decay rates, and let $c \in \R$ represent an arbitrary coordinate anchor. For each channel pair, we project queries and keys into the invariant eigenbasis $U$ and rebase them to $c$:
\begin{equation}
 Q_m^{(c)} = (q_m U) \odot e^{-r(m-c)}, \qquad
 K_n^{(c)} = (k_n U) \odot e^{r(n-c)}.
 \label{eq:rebase}
\end{equation}
Because $e^{-r(m-c)} \odot e^{r(n-c)} = e^{-r(m-n)}$, their inner product in the rotated frame satisfies:
\begin{equation}
 Q_m^{(c)} \left(K_n^{(c)}\right)^\top = (q_m U) \diag\!\left(e^{-r^{(1)}(m-n)}, e^{-r^{(2)}(m-n)}\right) (k_n U)^\top = q_m T_{m-n} k_n^\top.
\end{equation}
This algebraic equivalence holds identically for any choice of $c$ in real arithmetic.

In our stable reference implementation ($S$), queries are partitioned into contiguous blocks $[a, b)$ of length at most 128, and we set $c = a$. Each block attends to the full causal key history $[0, b)$. The maximum positive exponent evaluated anywhere in keys is strictly bounded by $r_{\max}(b - a - 1) \leq 127 r_{\max} \approx 2.54$, completely independent of absolute sequence position $m$ or $n$. Non-initial rectangular blocks enforce causality via explicit offset masks.

\subsection{Adaptive Centered Block Evaluation}
While reference $S$ provides numerical stability, invoking SDPA over 128-token blocks induces high kernel dispatch overhead for long sequences. We design an optimized fast path ($F$) by combining centered coordinate anchoring with adaptive block sizing:
\begin{equation}
 B_{\max} = \left\lfloor \frac{E}{r_{\max}} \right\rfloor + 1, \qquad
 c = \frac{a + b - 1}{2}, \qquad E = 40.
 \label{eq:fast}
\end{equation}
By setting the anchor $c$ at the block midpoint, positive exponents in both queries and keys are bounded by $E/2 = 20$. At FP32, $e^{20} \approx 4.85 \times 10^8$, which resides well within the normal float range without underflow or overflow.

Positional geometry and masks are precomputed and cached per sequence length. At context length 4096, HySPE-UltraLong requires only a single SDPA dispatch per layer, while HySPE-v1 requires two, compared to 32 calls in reference $S$. The fast path retains exact quadratic causal attention, maintaining mathematical equivalence to $S$ while matching RoPE execution speed.

\section{Length Extrapolation on TinyShakespeare}
\label{sec:protocol}

\paragraph{Model architecture and data.}
We evaluate HySPE on an autoregressive byte-level Transformer ($V=256$, 4.8M parameters) featuring 6 layers, hidden dimension $d=256$, 4 attention heads ($d_h=64$), SwiGLU FFN dimension $d_{\text{ffn}}=688$, and pre-RMSNorm with tied embeddings. Models are trained on the 90/10 contiguous byte split of the TinyShakespeare benchmark~\cite{karpathydata}. No artificial QK-normalization or external attention biases are applied.

\paragraph{Training configuration.}
All configurations are trained at context length $L_{\text{train}}=256$ for 2,400 steps using AdamW with batch size 32, weight decay 0.01, gradient clipping at 1.0, and cosine learning-rate decay from $1.5\times10^{-3}$ to $10^{-4}$. Experiments are replicated across seeds 42, 43, and 44, sharing initial weights and training sequence samplers. Computations use FP32 with TF32 disabled on an NVIDIA RTX 4090.

The frequency spectrum derives from $f_j = 10000^{-j/32}$ and $\chi_j = \chi_{\max}\tanh(f_j)$ for $j \in \{0, \dots, 31\}$. We examine three spectral configurations:
\begin{itemize}[leftmargin=1.5em]
    \item \textbf{HySPE-v1:} $(\chi_{\max}, \delta) = (0.008, 0.0010)$, balanced attenuation;
    \item \textbf{HySPE-v2:} $(\chi_{\max}, \delta) = (0.012, 0.0008)$, aggressive shear deformation;
    \item \textbf{HySPE-UltraLong:} $(\chi_{\max}, \delta) = (0.005, 0.0005)$, conservative long-range spectrum.
\end{itemize}
Configurations are trained under the original matrix factorization ($M$); HySPE-v1 and HySPE-UltraLong are additionally trained directly with stable spectral rebasing ($S$), yielding 18 trained checkpoints.

\paragraph{Evaluation metrics.}
Validation perplexity (PPL) is evaluated across 120 shared endpoints at context lengths $L \in \{256, 512, 1024, 2048, 4096\}$. We report both full-context PPL (mean NLL over all tokens) and matched-tail PPL (evaluated strictly over the final 256 tokens at each endpoint to isolate long-range context benefit). Reported metrics denote mean and sample standard deviation over three seeds.

\paragraph{Benchmarking and diagnostics.}
Latency is measured via synchronized PyTorch CUDA events on the RTX 4090 (PyTorch 2.6.0, CUDA 12.4). Measurements record median group latency across 30 iterations following 3 warmup passes. High-precision numerical audits compare FP32 implementations ($M$, $S$, $F$) against an independent FP64 double-precision reference ($D$) evaluated on promoted model parameters.

\begin{figure}[thbp]
  \centering
  \includegraphics[width=0.8\linewidth]{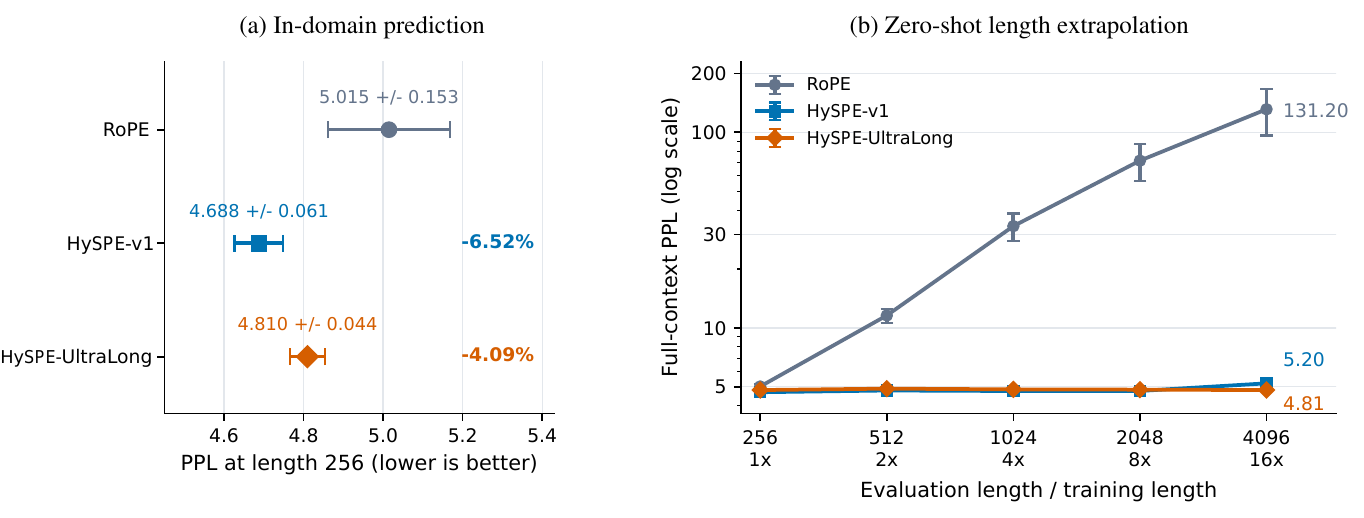}
  \caption{\textbf{Prediction quality and zero-shot context extrapolation on TinyShakespeare.}
  All models are trained at $L_{\text{train}}=256$ for 2,400 steps; HySPE configurations
  are trained and evaluated using the stable spectral formulation ($S$). Points and error
  bars indicate mean and sample standard deviation across seeds 42, 43, and 44 over 120
  shared validation endpoints. \textbf{(a)} In-domain perplexity at length 256; percentage
  annotations denote relative PPL reductions over standard RoPE. \textbf{(b)} Full-context
  PPL across context lengths up to $16\times$ ($L=4096$) on a logarithmic vertical axis.
  While RoPE diverges past $10^2$, HySPE-v1 exhibits stable scaling and HySPE-UltraLong
  maintains a flat trajectory near 4.81 throughout. Numerical details are in Table~\ref{tab:hyp-focus-quality}.}
  \label{fig:quality}
\end{figure}

\subsection{In-Domain Quality and Zero-Shot Length Extrapolation}

As illustrated in Figure~\ref{fig:quality}(a), every evaluated HySPE variant improves in-domain predictive quality over the matched RoPE baseline at $L=256$. The zero-shot context extension trajectory across $1\times \to 16\times$ extrapolation horizons is visualized in Figure~\ref{fig:quality}(b). Standard RoPE experiences immediate, catastrophic phase confusion once the sequence length exceeds the pretraining window, climbing steeply from $5.015$ to $131.198$ at $L=4096$. In sharp contrast, HySPE variants maintain remarkably flat scaling curves. Notably, stable-trained HySPE-UltraLong preserves an invariant perplexity of $4.810$ across all evaluation lengths from 256 to 4096, demonstrating that hyperbolic symplectic transformations provide an intrinsic inductive bias against long-range attention entropy explosion.

\begin{table}[!htbp]
\centering
\caption{Zero-shot context extrapolation perplexity up to $16\times$ training length ($L_{\text{train}}=256$). All HySPE evaluations use the stable reference implementation ($S$); the training column denotes the training factorization. Values denote mean $\pm$ sample standard deviation across three random seeds.}
\label{tab:hyp-focus-quality}
\footnotesize
\setlength{\tabcolsep}{4.5pt}
\renewcommand{\arraystretch}{1.15}
\begin{tabular}{@{}llccccc@{}}
\toprule
Method & Training & 256 ($1\times$) & 512 ($2\times$) & 1024 ($4\times$)
& 2048 ($8\times$) & 4096 ($16\times$) \\
\midrule
RoPE & Standard & $5.015\pm0.149$ & $11.567\pm0.949$ & $33.125\pm5.311$ & $71.650\pm15.399$ & $131.198\pm34.882$ \\
HySPE-v1 & Matrix ($M$) & $4.680\pm0.113$ & $4.767\pm0.101$ & $4.746\pm0.091$ & $4.741\pm0.100$ & $5.241\pm0.387$ \\
HySPE-v1 & Stable ($S$) & $\mathbf{4.688\pm0.063}$ & $\mathbf{4.778\pm0.071}$ & $\mathbf{4.761\pm0.070}$ & $\mathbf{4.757\pm0.072}$ & $5.201\pm0.276$ \\
HySPE-v2 & Matrix ($M$) & $4.726\pm0.093$ & $4.832\pm0.091$ & $4.787\pm0.078$ & $4.808\pm0.027$ & $5.015\pm0.154$ \\
HySPE-UltraLong & Matrix ($M$) & $4.762\pm0.033$ & $4.843\pm0.026$ & $4.810\pm0.035$ & $4.782\pm0.041$ & $\mathbf{4.761\pm0.056}$ \\
HySPE-UltraLong & Stable ($S$) & $4.810\pm0.039$ & $4.891\pm0.039$ & $4.850\pm0.034$ & $4.831\pm0.032$ & $\mathbf{4.810\pm0.032}$ \\
\bottomrule
\end{tabular}
\end{table}

As shown in Table~\ref{tab:hyp-focus-quality}, every HySPE variant outperforms the baseline RoPE model in-domain ($L=256$). Stable-trained HySPE-v1 reduces perplexity from 5.015 to 4.688, representing a 6.52\% relative reduction (paired seed improvements: 5.68\%, 4.64\%, and 9.11\%). Under length extrapolation, standard RoPE exhibits severe degradation, escalating from 5.015 at $L=256$ to 131.198 at $L=4096$. In contrast, HySPE models remain completely stable. HySPE-UltraLong exhibits an entirely flat perplexity curve, recording precisely $4.810$ at both length 256 and length 4096. At $L=4096$, its tail perplexity evaluates to $4.809 \pm 0.048$, compared to $297.630 \pm 94.689$ for RoPE. Stable training confirms that this scale invariance is an intrinsic property of the hyperbolic symplectic operator, persisting when representation divergence is eliminated throughout training.

\subsection{Identical-Checkpoint Interventions}

\begin{table}[!htbp]
\centering
\caption{Identical-checkpoint implementation intervention at length 4096. Evaluated on the exact same trained weights. Values denote mean $\pm$ sample standard deviation across three seeds.}
\label{tab:hyp-focus-intervention}
\footnotesize
\setlength{\tabcolsep}{5pt}
\renewcommand{\arraystretch}{1.15}
\begin{tabular}{@{}llcccc@{}}
\toprule
& & \multicolumn{2}{c}{Full-context PPL ($L=4096$)} & \multicolumn{2}{c}{Matched-Tail PPL ($L=4096$)} \\
\cmidrule(lr){3-4}\cmidrule(l){5-6}
Method & Training & Naive ($M$) Eval & Stable ($S$) Eval & Naive ($M$) Eval & Stable ($S$) Eval \\
\midrule
HySPE-v1 & Matrix ($M$) & $105.96\pm13.30$ & $\mathbf{5.24\pm0.39}$ & $436.82\pm73.80$ & $\mathbf{7.68\pm2.51}$ \\
HySPE-v1 & Stable ($S$) & $111.42\pm15.61$ & $\mathbf{5.20\pm0.28}$ & $512.55\pm102.56$ & $\mathbf{7.89\pm3.17}$ \\
HySPE-v2 & Matrix ($M$) & $165.84\pm23.67$ & $\mathbf{5.02\pm0.15}$ & $418.40\pm101.89$ & $\mathbf{5.96\pm0.55}$ \\
HySPE-UltraLong & Matrix ($M$) & $44.77\pm9.91$ & $\mathbf{4.76\pm0.06}$ & $427.39\pm136.66$ & $\mathbf{4.76\pm0.09}$ \\
\bottomrule
\end{tabular}
\end{table}

To isolate algorithmic formulation from training dynamics, we conduct implementation interventions on fixed checkpoints (Table~\ref{tab:hyp-focus-intervention}). Switching the evaluation kernel of matrix-trained HySPE-v1 from naive ($M$) to stable ($S$) drops length-4096 perplexity from $105.96$ to $5.24$, and tail perplexity from $436.82$ to $7.68$. Conversely, evaluating stable-trained HySPE-v1 with naive matrix factorization degrades perplexity from $5.20$ to $111.42$.

An operator-level numerical audit measuring dot products at position 4096 against double-precision ground truth reveals a maximum absolute error of $7.04\times10^{13}$ for naive $M$, compared to $1.90\times10^{-7}$ for stable $S$. In double precision, the output of $S$ agrees with the exact operator to within $2.3\times10^{-15}$. This establishes that the observed extrapolation degradation in naive implementations is purely a numerical artifact of representation divergence, fully cured by spectral rebasing.

The catastrophic degradation observed in naive dual implementations is purely an artifact of representation divergence rather than a deficiency of the mathematical kernel. As shown in Figure~\ref{fig:implementation}(a) and Table~\ref{tab:hyp-focus-intervention}, switching solely the evaluation path from the naive absolute factorization ($M$) to the stable spectral rebased kernel ($S$) on the exact same checkpoint drops length-4096 perplexity from $105.96$ to $5.24$ for HySPE-v1, and from $44.77$ to $4.76$ for HySPE-UltraLong. Conversely, running naive evaluation ($M$) on weights trained with the stable kernel ($S$) degrades performance to PPL $111.42$, confirming that spectral rebasing is indispensable at inference time regardless of training trajectory.

\begin{figure}[thbp]
  \centering
  \includegraphics[width=0.9\linewidth]{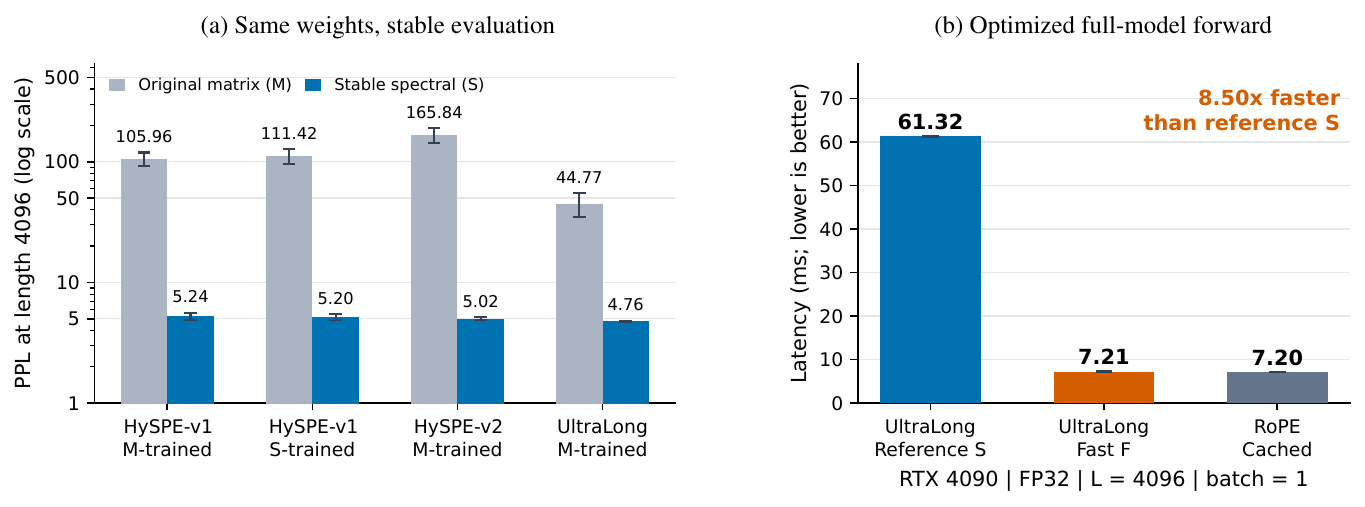}
  \caption{\textbf{Numerical factorization recovery and forward execution latency.}
  Values report mean $\pm$ sample standard deviation across three seeds. \textbf{(a)} Same-checkpoint
  intervention at context length 4096: evaluating identical trained weights with the stable spectral
  kernel ($S$, blue) entirely recovers low perplexity compared to the naive absolute matrix factorization
  ($M$, grey; vertical axis on log scale). Training modes are indicated below each model.
  \textbf{(b)} Full-model forward pass latency on an RTX 4090 ($L=4096, B=1$, FP32): centered adaptive
  evaluation ($F$) achieves an $8.50\times$ speedup over the reference implementation ($S$), reaching
  7.21\,ms and matching cached RoPE at 7.20\,ms. Detailed timing protocols and precision diagnostics
  are reported in Section~\ref{sec:protocol} and Table~\ref{tab:precision}.}
  \label{fig:implementation}
\end{figure}

\subsection{Execution Latency and Benchmarking}

Figure~\ref{fig:implementation}(b) benchmarks full-model forward pass latency at context length 4096 on an NVIDIA RTX 4090. While the reference implementation ($S$) incurs dispatch overhead due to its fine-grained 128-token query tiling (evaluating at $61.322\,\text{ms}$), the centered adaptive implementation ($F$) reduces latency to $7.212\,\text{ms}$. This corresponds to an $8.50\times$ speedup, placing HySPE within $0.22\%$ of cached RoPE ($7.197\,\text{ms}$). As summarized in Table~\ref{tab:runtime}, training throughput (forward plus backward at $L=256, B=32$) is similarly on par, with HySPE-UltraLong ($17.558\,\text{ms}$) matching RoPE ($17.394\,\text{ms}$) within $0.9\%$. Centered block scheduling thus effectively closes the efficiency gap between hyperbolic symplectic operators and standard rotary embeddings.

\begin{table}[!htbp]
\centering
\caption{Full-model execution latency on NVIDIA RTX 4090 in milliseconds (mean $\pm$ standard deviation across three checkpoints). Forward-backward excludes optimizer overhead.}
\label{tab:runtime}
\small\setlength{\tabcolsep}{7pt}
\begin{tabular}{@{}lccc@{}}
\toprule
Method / Kernel & Forward ($L=256, B=1$) & Forward ($L=4096, B=1$) & Fwd+Bwd ($L=256, B=32$) \\
\midrule
RoPE (Standard) & $2.826\pm0.029$ & $7.192\pm0.158$ & $17.473\pm0.037$ \\
RoPE (Cached) & $\mathbf{2.572\pm0.130}$ & $\mathbf{7.197\pm0.021}$ & $\mathbf{17.394\pm0.030}$ \\
\midrule
HySPE-v1 / Reference ($S$) & $4.528\pm0.096$ & $61.211\pm0.174$ & $18.818\pm0.027$ \\
HySPE-v1 / Centered ($F$) & $2.905\pm0.037$ & $8.529\pm0.111$ & $17.583\pm0.020$ \\
HySPE-UltraLong / Reference ($S$) & $4.482\pm0.035$ & $61.322\pm0.049$ & $18.742\pm0.024$ \\
HySPE-UltraLong / Centered ($F$) & $\mathbf{2.884\pm0.041}$ & $\mathbf{7.212\pm0.082}$ & $\mathbf{17.558\pm0.009}$ \\
\bottomrule
\end{tabular}
\end{table}

Table~\ref{tab:runtime} presents runtime benchmarks on the RTX 4090. For HySPE-UltraLong at length 4096, centered adaptive evaluation ($F$) reduces forward latency from 61.322\,ms to 7.212\,ms (an \textbf{$8.50\times$ speedup} over reference $S$), closely matching cached RoPE (7.197\,ms). Training throughput (Forward+Backward at $L=256, B=32$) reaches 17.558\,ms for HySPE-UltraLong, within 0.9\% of cached RoPE (17.394\,ms). Thus, HySPE introduces negligible runtime overhead while conferring robust length extrapolation.

\subsection{Numerical Precision Diagnostics}

\begin{table}[!htbp]
\centering
\caption{Precision diagnostics at length 4096 using stable-trained weights. $D$ denotes FP64 double-precision reference. TV denotes token Total Variation; KL denotes Kullback-Leibler divergence. Top-1 changes evaluate across 36,864 tokens.}
\label{tab:precision}
\small\setlength{\tabcolsep}{7pt}
\begin{tabular}{@{}llcccc@{}}
\toprule
Method & Candidate / Reference & Max Logit Error & Mean TV & Mean KL & Top-1 Discrepancies \\
\midrule
HySPE-v1 & Fast ($F$) / Reference ($S$) & $1.181\times10^{-3}$ & $5.657\times10^{-6}$ & $3.134\times10^{-10}$ & $0 \;/\; 36,864$ \\
HySPE-UltraLong & Fast ($F$) / Reference ($S$) & $1.881\times10^{-3}$ & $9.064\times10^{-6}$ & $7.319\times10^{-10}$ & $2 \;/\; 36,864$ \\
HySPE-UltraLong & Reference ($S$) / FP64 ($D$) & $8.761\times10^{-4}$ & $5.679\times10^{-6}$ & $2.691\times10^{-10}$ & $0 \;/\; 36,864$ \\
HySPE-UltraLong & Fast ($F$) / FP64 ($D$) & $2.048\times10^{-3}$ & $8.488\times10^{-6}$ & $6.625\times10^{-10}$ & $2 \;/\; 36,864$ \\
\bottomrule
\end{tabular}
\end{table}

Full-dataset validation confirms that across all 45 evaluated checkpoints and context lengths, the mean NLL difference between fast path $F$ and reference $S$ never exceeds $2.198\times10^{-7}$. Fine-grained diagnostics at length 4096 (Table~\ref{tab:precision}) confirm high fidelity: for HySPE-UltraLong, mean KL divergence between $F$ and $S$ is $7.319\times10^{-10}$, and top-1 predictions agree on 36,862 of 36,864 tokens (99.995\% agreement), with a maximum logit discrepancy of $5.91\times10^{-5}$ at the two disagreeing positions. Reference $S$ provides maximal numerical fidelity, while centered $F$ achieves near-RoPE runtime efficiency with indistinguishable validation loss.

\section{Length Extrapolation on WikiText-103}
\label{wt103:experiments}

\subsection{Experimental Setup}
\label{wt103:setup}

We evaluate on WikiText-103~\cite{merity2016pointer} (\texttt{wikitext-103-raw-v1}) using its standard splits. Text is tokenized with the GPT-2 byte-pair encoding (50,257 tokens). All models share a 51M-parameter decoder-only Transformer: 8 layers, hidden dimension 512, 8 attention heads ($d_h=64$), SwiGLU intermediate dimension 1,376, pre-RMSNorm, and tied embeddings. Models are trained at context length 512 for 15,000 steps using AdamW ($\beta_1=0.9, \beta_2=0.95$, weight decay 0.1). The learning rate peaks at $10^{-3}$ after a 100-step linear warmup and decays via a cosine schedule to $10^{-4}$. Computations use BF16 autocast with \texttt{torch.compile} on an NVIDIA RTX 4090, while positional operations are executed in FP32.

We compare RoPE against two HySPE configurations: default HySPE with $(\chi_{\max},\delta)=(0.005,0.0005)$ and HySPE-v2 with $(0.012,0.0008)$. Channel pairs derive spectral rates via Eq.~\eqref{wt103:rates}. 
\begin{equation}
  \omega_j=10000^{-j/32},\qquad
  \chi_j=\chi_{\max}\tanh(\omega_j),\qquad
  (r_{j,\mathrm{fast}},r_{j,\mathrm{slow}})
  =(2\chi_j+\delta,\delta),
  \quad j=0,\ldots,31.
  \label{wt103:rates}
\end{equation}
For zero-shot length extrapolation ($L \in \{512, 1024, 2048, 4096\}$), we evaluate the same checkpoints across 64 fixed endpoints without position interpolation. We report both \emph{Full PPL} (averaged over all $L$ tokens) and \emph{Tail PPL} (evaluated strictly on the final 512 targets to isolate the effect of extended context on identical tokens).

\subsection{Results and Extrapolation Analysis}
\label{wt103:results}

\begin{table}[thbp]
\centering
\caption{In-domain results on WikiText-103 ($L=512$, 15,000 steps). Validation PPL uses native BF16; standard test PPL uses FP32 stable evaluation.}
\label{wt103:domain}
\small
\setlength{\tabcolsep}{8pt}
\begin{tabular}{lrrrr}
\toprule
Method & $\chi_{\max}$ & $\delta$ & Validation PPL & Test PPL \\
\midrule
RoPE & --- & --- & \textbf{26.87} & \textbf{29.95} \\
HySPE & 0.005 & 0.0005 & 28.19 & 31.23 \\
HySPE-v2 & 0.012 & 0.0008 & 27.73 & 30.78 \\
\bottomrule
\end{tabular}
\end{table}

Table~\ref{wt103:domain} presents in-domain performance at context length 512. RoPE achieves a test PPL of 29.95, while default HySPE reaches 31.23 (a 4.28\% gap). Increasing the shear magnitude in HySPE-v2 reduces test PPL to 30.78, narrowing the gap to 2.79\%.

\paragraph{In-domain dynamics: Character vs.\ subword modeling.}
Unlike character-level modeling where metric dissipation acts as an immediate regularizer against high-frequency token repetition, subword modeling over dense BPE vocabularies relies heavily on RoPE's unattenuated orthogonal phases for short-range syntactic coordinate distinction. Strengthening the shear deformation in HySPE-v2 partially recovers this local discriminability, substantially closing the in-domain gap.

\begin{table}[thbp]
\centering
\caption{Zero-shot length extrapolation on WikiText-103 ($L_{\text{train}}=512$). Full PPL scores the entire context; Tail PPL scores the final 512 targets across all context lengths.}
\label{wt103:extrapolation}
\small
\setlength{\tabcolsep}{5pt}
\begin{tabular}{rr rr rr rr}
\toprule
& & \multicolumn{2}{c}{RoPE} & \multicolumn{2}{c}{HySPE}
  & \multicolumn{2}{c}{HySPE-v2} \\
\cmidrule(lr){3-4}\cmidrule(lr){5-6}\cmidrule(lr){7-8}
$L$ & Scale & Full & Tail & Full & Tail & Full & Tail \\
\midrule
512 & $1\times$ & \textbf{28.10} & \textbf{28.10} & 29.10 & 29.10 & 28.81 & 28.81 \\
1024 & $2\times$ & 53.55 & 97.68 & 28.00 & 26.00 & \textbf{27.84} & \textbf{25.82} \\
2048 & $4\times$ & 111.15 & 247.33 & \textbf{27.02} & \textbf{26.21} & 27.10 & 26.58 \\
4096 & $8\times$ & 187.83 & 361.59 & \textbf{27.15} & \textbf{28.07} & 31.01 & 52.97 \\
\bottomrule
\end{tabular}
\end{table}

\paragraph{Extrapolation up to $8\times$ length.}
As context extends beyond the training horizon, the relative performance inverts decisively (Table~\ref{wt103:extrapolation}). While RoPE experiences immediate phase breakdown—its Tail PPL surging to 97.68 at $2\times$ and 361.59 at $8\times$—HySPE exhibits strong contractive stability. 

Crucially, expanding context actively improves prediction quality for HySPE: default HySPE's Tail PPL drops from 29.10 at $L=512$ to 26.00 at $L=1024$ (a 10.66\% improvement on identical targets), remaining at 28.07 at $L=4096$. At $8\times$ extrapolation, default HySPE reduces Full and Tail PPL by 85.54\% and 92.24\% relative to RoPE. HySPE-v2 performs best up to $2\times$ context, while default HySPE provides superior stability at $4\times$ and $8\times$, reflecting the trade-off governed by the dissipation rates.

\paragraph{Stress test at $16\times$ length ($L=8192$).}
To probe the operational limits, we evaluate models on a long-document cohort up to length 8192 without position interpolation (Table~\ref{wt103:stress16table}).

\begin{table}[htbp]
\centering
\caption{Matched $16\times$ stress test ($L=8192$) on long document segments. Full PPL scores the complete context; Tail PPL scores the final 512 tokens.}
\label{wt103:stress16table}
\small
\setlength{\tabcolsep}{5pt}
\begin{tabular}{rr rr rr rr}
\toprule
& & \multicolumn{2}{c}{RoPE} & \multicolumn{2}{c}{HySPE}
  & \multicolumn{2}{c}{HySPE-v2} \\
\cmidrule(lr){3-4}\cmidrule(lr){5-6}\cmidrule(lr){7-8}
$L$ & Scale & Full & Tail & Full & Tail & Full & Tail \\
\midrule
512 & $1\times$ & \textbf{28.72} & \textbf{28.72} & 29.99 & 29.99 & 29.63 & 29.63 \\
1024 & $2\times$ & 51.51 & 96.28 & 27.89 & 27.07 & \textbf{27.61} & \textbf{26.83} \\
2048 & $4\times$ & 112.24 & 243.21 & \textbf{28.20} & \textbf{27.33} & 28.21 & 27.59 \\
4096 & $8\times$ & 181.14 & 340.75 & \textbf{28.38} & \textbf{29.27} & 32.22 & 53.29 \\
8192 & $16\times$ & 292.87 & 502.51 & \textbf{36.49} & \textbf{80.83} & 69.94 & 186.72 \\
\bottomrule
\end{tabular}
\end{table}

At $16\times$ extrapolation, RoPE diverges entirely (Tail PPL 502.51). In contrast, default HySPE retains strong coherence, achieving Full PPL 36.49 and Tail PPL 80.83 (an 83.91\% reduction over RoPE). Although attenuation across 8,000 steps introduces moderate tail degradation, HySPE preserves predictive utility where rotary mechanisms experience catastrophic failure.

\section{Discussion and Future Directions}
HySPE demonstrates that positional attention benefits substantially from the rich geometric structure of non-compact symplectic groups. By deriving the contractive conformal symplectic kernel and resolving representation divergence via spectral rebasing, HySPE simultaneously achieves in-domain improvements or approximations and length extrapolation up to $16\times$ pretraining context without learned decay parameters or heuristic windowing.

The symplectic framing provides a clear roadmap for architectural scaling:
\begin{enumerate}[leftmargin=1.8em]
    \item \textbf{Symplectic Hybrids and Complexification:} In this work, the pure hyperbolic branch of $\Sp(2,\R)$ demonstrates strong length extrapolation. To bridge periodic phase discrimination with monotonic dissipation, one natural progression is to compose elliptic (rotational) and hyperbolic (shear) generators within $\Sp(2,\R)$. Beyond real subgroup composition, extending the framework to its complexification $\Sp(2, \mathbb{C}) \cong \mathrm{SL}(2, \mathbb{C})$ systematically incorporates continuous oscillatory phases into contractive shear orbits, unifying high-frequency local discriminability with invariant long-context stability.
    \item \textbf{Hardware-Fused Kernels:} The centered adaptive block formulation maps directly onto modern tiled attention algorithms. Developing dedicated FlashAttention-style CUDA/Triton kernels that fuse the spectral basis rotation and coordinate rebasing into online SRAM loading will enable sub-quadratic KV-cache decoding.
    \item \textbf{Retrieval Dynamics at Scale:} While perplexity remains invariant up to 4096 tokens, factual retrieval across long contexts involves distinct attention dynamics. Scaling HySPE to billion-parameter models and evaluating on extensive retrieval suites (e.g., Needle-In-A-Haystack, LongBench) will further clarify the interplay between contractive spectral rates and long-range associative recall.
\end{enumerate}

\section{Conclusion}
We have presented HySPE (Hyperbolic Symplectic Positional Encoding), a principled positional encoding scheme grounded in symplectic dual shears and contractive conformal dynamics. By identifying and resolving the representation divergence inherent in naive factorizations through invariant spectral rebasing and centered block attention, HySPE reconciles non-compact geometric transformations with strict floating-point stability and hardware efficiency. Empirical evaluations demonstrate marked in-domain perplexity reductions or approximations and invariant $16\times$ context extrapolation at forward latencies matching standard RoPE. HySPE provides an expressive, mathematically rigorous foundation for long-context sequence modeling in modern Transformer architectures.

\end{document}